\documentclass{article}

\usepackage{arxiv}

\usepackage[utf8]{inputenc} 
\usepackage[T1]{fontenc}    
\usepackage{hyperref}       
\usepackage{url}            
\usepackage{booktabs}       
\usepackage{amsfonts}       
\usepackage{nicefrac}       
\usepackage{microtype}      
\usepackage{lipsum}		
\usepackage{graphicx}
\usepackage{natbib}
\usepackage{doi}
\usepackage{comment}
\usepackage{amsmath}
\usepackage{ntheorem}
\newtheorem*{theorem}{Theorem}

\newtheorem*{proof}{Proof}
\usepackage{microtype}      
\usepackage{cleveref}       
\usepackage{lipsum}         
\usepackage{graphicx}
\usepackage{natbib}
\usepackage{doi}
\usepackage{wrapfig}
\usepackage{float}

\title{Learning Hamiltonian
Density Using DeepONet}

\author{%
	Baige Xu\\
	Graduate School of Science\\
	Kobe University\\
	Kobe, Japan \\
\texttt{baigexu@stu.kobe-u.ac.jp} \\
	\And
	Yusuke Tanaka \\
	NTT Communication Science Laboratories\\
	NTT Corporation\\
	Kyoto, Japan \\
\texttt{ysk.tanaka@ntt.com} \\
    \And
	Takashi Matsubara \\
	Graduate School of Information Science and Technology\\
	Hokkaido University\\
	Sapporo, Japan \\
\texttt{matsubara@ist.hokudai.ac.jp} \\
    \And
	Takaharu Yaguchi\\
	Graduate School of Science\\
	Kobe University\\
	Kobe, Japan \\
\texttt{yaguchi@pearl.kobe-u.ac.jp} \\
}

\renewcommand{\undertitle}{}
\renewcommand{\shorttitle}{Learning Hamiltonian 
Density Using DeepONet}

\hypersetup{
pdftitle={A template for the arxiv style},
pdfsubject={q-bio.NC, q-bio.QM},
pdfauthor={David S.~Hippocampus, Elias D.~Striatum},
pdfkeywords={First keyword, Second keyword, More},
}

\begin{document}
\maketitle

\begin{abstract}
In recent years, deep learning for modeling physical phenomena which can be described by partial differential equations (PDEs) have received significant attention. 
For example, for learning Hamiltonian mechanics, methods based on deep neural networks such as Hamiltonian Neural Networks (HNNs) and their variants have achieved progress.
However, 
existing methods typically
depend on the discretization of data, and the determination of required differential operators is often necessary. 
Instead, in this work, we propose an operator learning approach 
for modeling wave equations. In particular, we present a method to compute the variational derivatives that are needed to formulate the equations using the automatic differentiation algorithm.
The experiments demonstrated that 
the proposed method
is able to learn the operator 
that defines the Hamiltonian density of waves from data with unspecific discretization without 
determination of the differential operators.
\end{abstract}


\section{Introduction}
Many physical phenomena in nature can be described by partial differential equations (PDEs) \cite{pde}, such as diffusion in liquids and gases, thermal conduction, and wave propagation \cite{diff,thermal,wave}.
Modeling and solving these physically meaningful equations have practical significance 
in many applications like weather forecasting, material processing, and aircraft control \cite{weather,material,air}.
In recent years, the techniques of modeling and solving 
PDEs based on deep learning have been widely studied instead of conventional numerical solvers.
For example, for learning Hamiltonian mechanics, which provide a foundational framework of diverse physical phenomena, Hamiltonian Neural Networks (HNNs) and their variants have been actively developed 
as powerful tools for 
modeling equations \cite{hnn,symp}.
However, since these methods are based on deep neural networks, they depend on the discretization of the data. 
Also, the differential operators appearing in the equations, such as $(\partial_x,\partial_{xx},\cdots)$, 
and their discretization are necessary to be determined when learning the equations. 

To give an example of such equations, consider modeling Hamiltonian PDEs that describe 
many wave phenomena: 
\begin{equation*}
    \dfrac{\partial u}{\partial t}=\mathcal{D}
    \dfrac{\delta \mathcal{H}}{\delta u},
\end{equation*}
where $u=u(t,x)$ is a function defined on a Hilbert space $\mathcal{X}$, and $\mathcal{D}$ is a skew-symmetric operator, which includes specific examples such as
\begin{equation*}
    \left(
    \begin{array}{cc}
       0 & \mathrm{Id} \\
       -\mathrm{Id} & 0
    \end{array}
    \right),\quad
    \dfrac{\partial}{\partial x},\quad
    \left(
    \begin{array}{cc}
        0 & \dfrac{\partial}{\partial x} \\
    \dfrac{\partial}{\partial x} & 0
    \end{array}
    \right). 
\end{equation*}
$\mathcal{H}$ is the Hamiltonian, the energy functional of the equation, and is defined as $\mathcal{H}=\int H(u,u_x,u_{xx},\cdots)\mathrm{d}x$. The functional $H$ is called the Hamiltonian density of the equation, representing the local energy distribution within the dynamical system. $\dfrac{\delta \mathcal{H}}{\delta u}$ is called the variational derivative of $\mathcal{H}$, whose definition is given as
\begin{equation*}
\dfrac{\mathrm{d}}{\mathrm{d}\epsilon}\mathcal{H}[u+\epsilon v]|_{\epsilon = 0}=\langle \dfrac{\delta \mathcal{H}}{\delta u}, v \rangle _{L^2},\quad \forall u,v\in\mathcal{X},
\end{equation*}
where $\langle \cdot , \cdot \rangle _{L^2}$ represents the inner product on $\mathcal{X}$.

When modeling such equations using HNNs or other methods based on neural networks, usually, $H$ is modeled by using a neural network  $H_{\mathrm{NN}}$.
To train the neural network, it is necessary to determine differential variables such as $(u,u_x,u_{xx},\ldots)$ that appear in the Hamiltonian density 
and their discretizations, which is not easy without detailed knowledge of target phenomena. 
Therefore, such methods are not ideal in practice. 

To overcome this limitation, in this study, we propose an operator learning approach for modeling wave equations, that does not rely on the discretizations of data and does not require determining the variables involving differential operators or their discretizations.
In particular, we show that the variational derivative of $\mathcal{H}$ can be obtained from the gradient of $\mathcal{H}$ computed via automatic differentiation algorithm.
The results of numerical experiments show that the proposed method 
can in fact learn the Hamiltonian density of the wave equations. 

\section{Operator Learning and DeepONet}
Operator learning addresses the challenges faced by traditional neural networks by learning mappings between infinite-dimensional function spaces directly. Among operator learning methods, Neural Operators have gained significant attention for their ability to combine the universality of operator learning with the expressive power of deep learning architectures \cite{no}. One notable variant 
of neural operators is DeepONet, which is designed to efficiently learn and approximate arbitrary nonlinear continuous operators \cite{don}. 
DeepONet has demonstrated remarkable performance in solving parametric PDEs and other high-dimensional problems. 

Consider the general architecture of a DeepONet, when learning a simple dynamical system described by 
an ODE:
\begin{equation*}
    \dfrac{\mathrm{d}u(t)}{\mathrm{d}t}=f(t).
\end{equation*}
DeepONet represents the target operator $G:f \mapsto u$ as a pair of sub networks.
The first one is called the branch network, which processes the
values of the input function $u$ sampled at a specific finite set of location points $\{x_i\}_{i=1}^m$, taking $(f(x_1),f(x_2),...,f(x_m))^{\top}\in\mathbb{R}^m$ as input and $(b_1,\cdots,b_p)\in\mathbb{R}^p$ as output.
The second one is called trunk network, which evaluates the operator at the target locations. For an unstacked DeepONet, there is just one trunk network. It takes $y=(y_1,\cdots,y_n)\in\mathbb{R}^n$ as input and $(t_1,t_2,...,t_p)^{\top} \in \mathbb{R}^p$ as output. 
Then the outputs of two sub networks are merged together by taking their inner product 
to approximate the
target operator $G$:
\begin{equation*}
    G(f)(y) \simeq G_\mathrm{NO}(f)(y) = 
    \sum_{k=1}^{p} b_k t_k.
\end{equation*}
In fact, $G_\mathrm{NO}(f)(y)$ can be viewed as a function of $y$ conditioning on $u$. 
The loss function of DeepONet is given as 
\begin{equation*}
    \mathcal{L} = \int {\lvert \lvert G_{\mathrm{target}}(f)(y)-G_{\mathrm{model}}(f)(y)} \rvert \rvert ^2 \mathrm{d}y,
\end{equation*}
which is typically approximated by, e.g., 
\begin{equation*}
    \mathcal{L} \approx \sum_{y_i} {\lvert \lvert G_{\mathrm{target}}(f)(y_i)-G_{\mathrm{model}}(f)(y_i) \rvert \vert}^2,
\end{equation*}
where $y_i$'s are given collocation points. By minimizing the above loss function, DeepONet learns the target operator.

\section{Operator Learning for Modeling Hamiltonian Partial Differential Equations Using DeepONet}
In this section, we propose a learning method 
for modeling Hamiltonian partial differential equations using DeepONet. In particular, we propose a method to compute the 
variational derivative of 
Hamiltonian $\mathcal{H}$ that is required to determine the PDEs 
using the automatic differentiation algorithm.

\subsection{Learning Hamiltonian Density Using DeepONet}
In the proposed model, we use DeepONet to learn an operator that maps the state variable to the energy density.
The input of the branch network is 
the values $(u(x_1),\cdots,u(x_N))^\top$ of the function $u(t,x)$ at a finite set of position points $\{x_i\}_{i=1}^N$ as the input function to DeepONet.
The trunk network is given a vector $y$, which represents the points in the domain 
of the output function. The output of DeepONet is 
the inner product of the two sub network outputs $(c_1,\cdots,c_p)^\top \in \mathbb{R}^p$ and $(\psi_1,\cdots,\psi_p)^\top \in \mathbb{R}^p$, as
\begin{equation*}
    H_{NO} = \sum_{k=1}^{p} c_k \psi_k.
\end{equation*}
We train the networks so that $H_{NO}$ approximates the target Hamiltonian density.
By integrating the learned Hamiltonian density $H_{NO}$, we can obtain the learned Hamiltonian $\mathcal{H}_{NO}$.

To train the networks we need to compute
the variational derivatives $\dfrac{\delta \mathcal{H}_{NO}}{\delta u}$ of the Hamiltonian. 
Generally speaking, the variational derivatives of $\mathcal{H}_{NO}$ can be 
obtained by partial integration using
\begin{equation*}
    \dfrac{\delta \mathcal{H}_{NO}}{\delta u}=\dfrac{\partial H_{NO}}{\partial u} - \partial _x (\dfrac{\partial H_{NO}}{\partial u_x})+ \partial _x ^2 (\dfrac{\partial H_{NO}}{\partial u_{xx}})-\cdots;
\end{equation*}
however, in the model, $H_{NO}$ learned by DeepONet is not given as such a form in which this calculation can be performed.
Therefore, 
a method for finding variational derivatives must be newly designed. 

\subsection{Deriving Variational Derivatives of Hamiltonian Using Automatic Differentiation Algorithm}
To derive a method for finding variational derivatives, we need to consider the roles of the trunk and the branch networks. 
In fact, in the architecture of DeepONet, 
the trunk network provides a basis $(\psi_1,\cdots,\psi_p)$ for the target function space, and the branch network provides the coefficients $(c_1,\cdots,c_p)$ of the basis. 
In particular, $(c_1,\cdots,c_p)$ is obtained by introducing a basis $L_i$ of the dual space of subspaces of the input function space, where $L_i$ satisfies $L_i(u)=u(x_i)$.
Then, introducing a non-contradictory basis $\phi_j$ such that $L_i \phi_j=\delta_{ij}$, then the input function $u$ that is tied to the discrete value vector $(u(x_1),\cdots,u(x_N))^\top$ 
can be represented as
\begin{equation}\label{u}
  \begin{aligned}
    u&=(L_1(u))\phi_1 +\cdots+ (L_N(u))\phi_N 
    = u(x_1)\phi_1+\cdots+u(x_N)\phi_N.
  \end{aligned}
\end{equation}
It is important to notice that we can compute derivatives with respect to the vector $(u(x_1),\cdots,u(x_N))^\top$
by using the automatic differentiation algorithm. We want to associate the computed derivatives by this algorithm with the variational derivative.

On the other hand, let $\mathcal{H}:(\mathcal{X},\langle,\rangle _{L^2} )\to \mathbb{R}$ be a smooth enough Hamiltonian and for any $u \in \mathcal{X}$, 
\begin{equation*}
    v\mathcal{H}=\mathrm{d}\mathcal{H}(v),\quad \forall v \in \mathcal{T}_u\mathcal{X}
\end{equation*}
holds. That means $\mathrm{d}\mathcal{H}$ is defined as a linear functional $\mathrm{d}\mathcal{H}:v \in \mathcal{T}_u\mathcal{X} \mapsto \mathrm{d}\mathcal{H}(v)\in \mathbb{R}$ and $\mathrm{d}\mathcal{H} \in \mathcal{T}_u^*\mathcal{X}$.
$\mathcal{T}_u\mathcal{X}$ and $\mathcal{T}_u^*\mathcal{X}$ denote the tangent space and cotangent space on $\mathcal{X}$, respectively. 

By Riesz's representation theorem, there exists a $v_{\mathcal{H}}\in \mathcal{T}_u\mathcal{X}$ such that
\begin{equation}\label{grad}
    \mathrm{d}\mathcal{H}(w)= \langle v_{\mathcal{H}},w \rangle, \quad \forall w \in \mathcal{T}_u\mathcal{X}.
\end{equation}
Also, from the definition of $\dfrac{\delta \mathcal{H}}{\delta u}$, it holds that
\begin{equation}\label{riesz}
    \mathrm{d}\mathcal{H}(w)= \langle \dfrac{\delta \mathcal{H}}{\delta u},w \rangle _{L^2}, \quad \forall w \in \mathcal{T}_u\mathcal{X}.
\end{equation}
That means the gradient $V_{\mathcal{H}}$ defined by any inner product is related to the variational derivative using the $L^2$ inner product.
Using the above \eqref{u}, \eqref{grad} and \eqref{riesz}, the variational derivative of $\mathcal{H}_{NO}$ in DeepONet training can be obtained using the gradient obtained by automatic differentiation algorithm.

\begin{theorem}
Suppose that $\Tilde{\mathcal{X}}=\mathrm{span}\{\phi_1,...,\phi_N\}$ is a subspace of a Hilbert space $\mathcal{X}$, $\phi_i = \phi_i(x) \in \mathcal{X},i=1,...,N$. Let $\mathcal{H}_{NO}:(\Tilde{\mathcal{X}},\langle,\rangle _{L^2} ) \to \mathbb{R}$ be a sufficiently smooth Hamiltonian learned by DeepONet. If $\nabla_{AD}\mathcal{H}_{NO}$ is a gradient of $\mathcal{H}_{NO}$ which is derived through automatic differentiation, for any $u \in \mathcal{X}$ and $\Tilde{u} \in \Tilde{\mathcal{X}}$, the variational derivative of $\mathcal{H}_{NO}$, $\dfrac{\delta \mathcal{H}_{NO}}{\delta u}$, can be obtained from $\nabla_{AD}\mathcal{H}_{NO}$:
{\small
\begin{equation*}
\begin{aligned}
    \dfrac{\delta \mathcal{H}_{NO}}{\delta u}=
    (\nabla_{AD}\mathcal{H}_{NO})^\top
    \begin{pmatrix}
            \int \phi_1^2\mathrm{d}x & \cdots & \int \phi_1 \phi_N\mathrm{d}x\\
            \vdots & \ddots & \vdots \\
            \int\phi_1 \phi_N\mathrm{d}x & \cdots & \int \phi_N^2\mathrm{d}x
    \end{pmatrix}^{-1}
    \begin{pmatrix}
            \phi_1 \\
            \vdots \\
            \phi_N
    \end{pmatrix}.
\end{aligned}
\end{equation*}
}
\end{theorem}

\begin{proof}
For a continuous linear functional $\mathrm{d}\mathcal{H}_{NO} \in \mathcal{T}_u^*\Tilde{\mathcal{X}}$, there exists a $\dfrac{\delta \mathcal{H}_{NO}}{\delta u}=g_1\phi_1+\cdots+g_N\phi_N\in\mathcal{T}_u\Tilde{\mathcal{X}}$ such that for all $w=w_1\phi_1+\cdots+w_N\phi_N \in \mathcal{T}_u\Tilde{\mathcal{X}}$,
\begin{equation*}\label{l2}
\begin{aligned}
\mathrm{d}\mathcal{H}_{NO}(w)&=\langle \dfrac{\delta \mathcal{H}_{NO}}{\delta u},w\rangle_{L^2}
=\int (\dfrac{\delta \mathcal{H}_{NO}}{\delta u}w)\mathrm{d}x \\
&= \int ((g_1\phi_1 + \cdots + g_N\phi_N)(w_1\phi_1+\cdots+w_N\phi_1))\mathrm{d}x \\
&=(g1,\cdots,g_N)
    \begin{pmatrix}
        \int \phi_1^2\mathrm{d}x & \cdots & \int \phi_1 \phi_N\mathrm{d}x\\
            \vdots & \ddots & \vdots \\
            \int\phi_1 \phi_N\mathrm{d}x & \cdots & \int \phi_N^2\mathrm{d}x
    \end{pmatrix}
    \begin{pmatrix}
        w_1\\ \vdots \\w_N
    \end{pmatrix}.
\end{aligned}
\end{equation*}
In the subspace $\Tilde{\mathcal{X}}=\mathrm{span}\{\phi_1,...,\phi_N\}$, we can introduce the following two inner product: 
\begin{equation*}
    \langle \cdot , \cdot \rangle _{L^2}: \langle u,v \rangle _{L^2} = \int (uv)\mathrm{d}x,
\end{equation*}
\begin{equation*}
    \langle \cdot , \cdot \rangle _{\mathbb{R}^N}: \langle u,v \rangle _{\mathbb{R}^N} = \sum L_i(u)L_i(v) = \sum u_i v_i.
\end{equation*}
From Riesz's representation theorem,
\begin{equation*}
    \begin{aligned}
        \mathrm{d}\mathcal{H}_{NO}(w)&=
        \langle \dfrac{\delta \mathcal{H}_{NO}}{\delta u},w\rangle_{L^2}
        =\langle \nabla \mathcal{H}_{NO},w \rangle _{\mathbb{R}^N} \\
        &= \sum _{i=1}^N (\nabla \mathcal{H}_{NO})_i w_i \\
        &= (\nabla_{AD}\mathcal{H}_{NO})^\top
        \begin{pmatrix}
        w_1\\ \vdots \\w_N
        \end{pmatrix}.
    \end{aligned}
\end{equation*}
Therefore, we get
\begin{equation*}
\begin{aligned}
    &(g_1,\cdots,g_N)= (\nabla_{AD}\mathcal{H}_{NO})^\top
    \begin{pmatrix}
            \int \phi_1^2\mathrm{d}x & \cdots & \int \phi_1 \phi_N\mathrm{d}x\\
            \vdots & \ddots & \vdots \\
            \int\phi_1 \phi_N\mathrm{d}x & \cdots & \int \phi_N^2\mathrm{d}x
    \end{pmatrix}^{-1},
\end{aligned}
\end{equation*}
and hence
\begin{equation*}
\begin{aligned}
    \dfrac{\delta \mathcal{H}_{NO}}{\delta u}
    &= g_1\phi_1 + \cdots + g_N\phi_N& \\
    &=(\nabla_{AD}\mathcal{H}_{NO})^\top
    \begin{pmatrix}
            \int \phi_1^2\mathrm{d}x & \cdots & \int \phi_1 \phi_N\mathrm{d}x\\
            \vdots & \ddots & \vdots \\
            \int\phi_1 \phi_N\mathrm{d}x & \cdots & \int \phi_N^2\mathrm{d}x
    \end{pmatrix}^{-1}
    \begin{pmatrix}
            \phi_1 \\
            \vdots \\
            \phi_N
    \end{pmatrix}.
\end{aligned}
\end{equation*}
$\hfill\square$
\end{proof}

\section{Numerical Example: Wave Equations}
In this section, we demonstrate the effectiveness of our proposed learning method by using the linear wave equation 
$$
\dfrac{\partial^2 u}{\partial t^2} = \dfrac{\partial^2 u}{\partial x^2}
$$ 
as an example. The wave equation can be written in the Hamiltonian form as follows: 
\begin{equation*}
    \dfrac{\partial w}{\partial t} = 
    \begin{pmatrix}
        0 & 1 \\ -1 & 0
    \end{pmatrix}
    \dfrac{\delta \mathcal{H}}{\delta w},\quad w =
    \begin{pmatrix}
        u \\ u_t
    \end{pmatrix}.
\end{equation*}
\textbf{Data}
We considered a wave system whose Hamiltonian is known as $\mathcal{H}=\dfrac{1}{2}(u_x^2+u_t^2)$. The time-space region is set to 
$[0,2]\times[0,1]$. 
The initial condition and the boundary condition is set to $u(0,x)=f(x),\\u_t(0,x)=0,u(t,0)=u(t,1)$. Under these conditions, the solution of the wave equation is known as $u(t,x)=\dfrac{1}{2}(f(t+x)+f(t-x))$. 
We used 
a dataset of the solutions for 1000 different $f(x)$.

\textbf{Training}
We adopted two multi-layer perceptrons (MLPs) to implement DeepONet. The networks has 3 layers, 200 hidden units, and the tanh activation functions. We used the Adam optimizer to minimize the loss function of DeepONet. The learning rate is set to $10^{-4}$, and the epochs of iteration is 20000.

\textbf{Results}
Figure \ref{fig:fig1} shows the true solution $u$ and $u_t$, and Figure \ref{fig:fig2} shows $u_{NO}$ and ${u_t}_{NO}$ learned by our model. Figure \ref{fig:fig3} shows the behavior of the Hamiltonian. 
The solution predicted by DeepONet is very similar to the true solution, but they tend to diverge over time. Also, the Hamiltonian directly computed from the physical quantities exhibits similar behavior,  which may be due to the insufficient number of data. 
\begin{figure}[H]
	\begin{center}
	\includegraphics[width=0.7\textwidth]{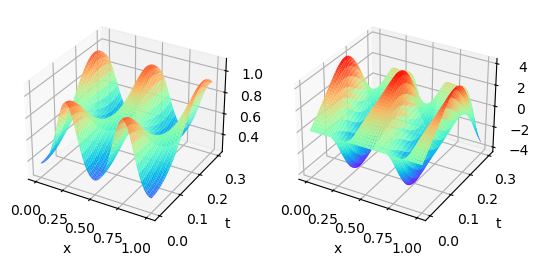}
	\end{center}
	\caption{True $u$ and $u_t$ of the equation.}
	\label{fig:fig1}
\end{figure}

\begin{figure}[H]
	\begin{center}
	\includegraphics[width=0.7\textwidth]{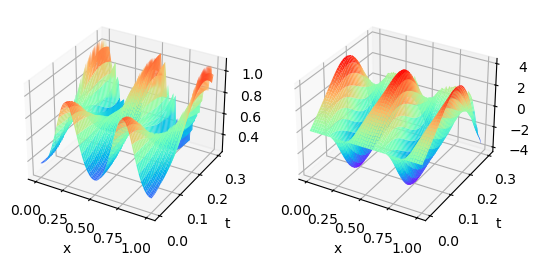}
	\end{center}
	\caption{Learned $u_{NO}$ and ${u_t}_{NO}$ of the equation.}
	\label{fig:fig2}
\end{figure}

\begin{figure}[H]
	\begin{center}
	\includegraphics[width=0.4\textwidth]{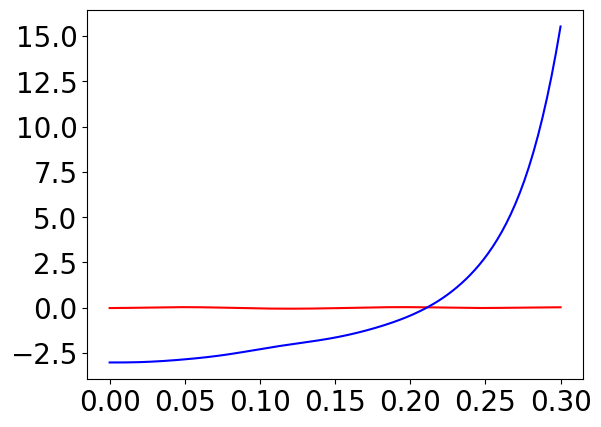}
	\end{center}
	\caption{The behavior of Hamiltonian. Red: $\mathcal{H}_{NO}(u,u_t)$, Blue: $\mathcal{H}(u_{NO},{u_t}_{NO})$.}
	\label{fig:fig3}
\end{figure}

\section{Conclusion}
We propose an operator learning approach for modeling wave equations. In particular, we present a method to compute the variational derivative of Hamiltonian from the gradient computed by the automatic differentiation algorithm. 
Experiments have demonstrated the effectiveness of our proposal in learning the operator of energy density. Further training with more data and higher-precision numerical integration are anticipated to improve the model.

\section*{Acknowledgments}
This work is supported by JST CREST Grant Number JPMJCR1914, JST ASPIRE JPMJAP2329 and JST SAKIGAKE JPMJPR21C7.

\bibliographystyle{plain}

\begin{thebibliography}{10}
%
\bibitem{pde}
Strauss W A. Partial differential equations: An introduction, John Wiley \& Sons, 2007.
\bibitem{diff}
Wu Z, Zhao J, Yin J, et al. Nonlinear diffusion equations, 2001.
\bibitem{thermal}
Carthel C, Glowinski R, Lions J L. On exact and approximate boundary controllabilities for the heat equation: a numerical approach, Journal of optimization theory and applications, 82: 429-484, 1994.
\bibitem{wave}
Durran D R. Numerical methods for wave equations in geophysical fluid dynamics, Springer Science and Business Media, 2013.
\bibitem{weather}
Lynch P. The origins of computer weather prediction and climate modeling, Journal of computational physics, 227(7): 3431-3444, 2008. 
\bibitem{material}
Mazumder J, Steen W M. Heat transfer model for CW laser material processing, Journal of Applied Physics, 51(2): 941-947, 1980.
\bibitem{air}
Stevens B L, Lewis F L, Johnson E N. Aircraft control and simulation: dynamics, controls design, and autonomous systems, John Wiley \& Sons, 2015.
\bibitem{hnn}
Samuel Greydanus, Misko Dzamba and Jason Yosinski, Hamiltonian neural networks, Advances in neural information processing systems, 32, 2019.
\bibitem{symp}
Jin P, Zhang Z, Zhu A, et al. SympNets: Intrinsic structure-preserving symplectic networks for identifying Hamiltonian systems, Neural Networks, 132: 166-179, 2020.
\bibitem{no}
Kovachki, Nikola, et al. Neural operator: Learning maps between function spaces with applications to pdes, Journal of Machine Learning Research 24.89 : 1-97, 2023.
\bibitem{don}
Lu Lu, Pengzhan Jin, Guofei Pang, Zhongqiang Zhang and George Em Karniadakis, Learning nonlinear operators via DeepONet based on the universal approximation theorem of operators, Nature machine intelligence, 3(3), 218-229, 2021.
%
%
%
%
\end{thebibliography}

\end{document}